\documentclass{article}

\usepackage{spconf}
\usepackage{amsmath}
\usepackage{graphicx}
\usepackage{bm}
\usepackage{amsfonts}
\usepackage{mathtools}
\usepackage{adjustbox}
\usepackage{tabularx}
\usepackage{multirow}
\usepackage{makecell}
\usepackage{hhline}
\usepackage{tikz}
\usepackage{colortbl}
\usepackage{alphalph}
\usepackage{caption}
\usepackage{subcaption}
\usepackage{amsthm}
\usepackage{siunitx}
\usepackage{hyperref}
\usepackage{url}
\usepackage{cleveref}
\usepackage{fancyhdr}

\newtheorem{theorem}{Theorem}

\usetikzlibrary{arrows.meta}
\usetikzlibrary{positioning}

\tikzstyle{block} = [rectangle, draw=black, minimum height=1cm]
\tikzstyle{vector} = [rectangle, draw=black, minimum height=3cm]

\def\x{{\mathbf x}}
\def\y{{\mathbf y}}
\def\m{{\mathbf m}}
\def\b{{\mathbf b}}

\def\W{{\mathbf W}}

\def\M{{\mathbf M}}
\def\C{{\mathbf C}}
\def\d{{\mathbf d}}
\def\E{{\text E}}

\def\Real{{\mathbb R}}
\def\pmt{{\(\pm\)}}

\title{Regularizing Neural Networks by Stochastically Training Layer Ensembles}
\name{Alex Labach and Shahrokh Valaee}
\address{University of Toronto\\
	Department of Electrical and Computer Engineering\\
	Toronto, Canada}

\begin{document}

\thispagestyle{fancy}
\maketitle

\begin{abstract}
	Dropout and similar stochastic neural network regularization methods are often interpreted as implicitly averaging over a large ensemble of models. We propose STE (stochastically trained ensemble) layers, which enhance the averaging properties of such methods by training an ensemble of weight matrices with stochastic regularization while explicitly averaging outputs. This provides stronger regularization with no additional computational cost at test time. We show consistent improvement on various image classification tasks using standard network topologies.
\end{abstract}


\section{Introduction}
\label{sec:intro}

In order to generalize well to new inputs, modern deep neural networks require heavy regularization. While many techniques for achieving this have been proposed, dropout~\cite{hinton2012improving} and related methods have become some of the most widely used approaches~\cite{labach2019survey}. Dropout works by randomly removing neurons from a neural network with a certain probability at each training step, then using the full network at test time. It is commonly interpreted as implicitly averaging over an ensemble of networks, where the ensemble contains all possible networks obtained by removing a subset of neurons~\cite{hinton2012improving,wardefarley2014empirical,baldi2013understanding,bachman2014learning}. The regularization power of dropout then comes from this averaging, where individual elements of the network ensemble may be overfitted, but averaging them reduces this effect.

We introduce STE (stochastically trained ensemble) layers, which add explicit averaging to dropout and related regularization methods~\footnote{Sample code can be found at \url{https://github.com/j201/keras-ste-layers}}. In place of a standard dense neural network layer, we propose using multiple different weight matrices and bias vectors to transform inputs. Dropout or a similar method is applied to all resulting values, and corresponding outputs are then averaged together before an activation function is applied, optionally followed by dropout again. Our results suggest that training multiple sets of weights and biases together in this fashion along with dropout provides stronger regularization than dropout alone.

Our proposed method is related to various established ensemble methods in machine learning. This includes bagging, boosting, and mixture of expert methods, where multiple instances of a machine learning model are each trained on a subset of the input space or training set and collaborate during inference~\cite{dietterich2000ensemble,masoudnia2014mixture}. While our method also stochastically trains multiple weight matrices on different subsets of the training set, after training, averaging is applied to produce a single weight matrix for use in inference. It therefore functions as a regularization method for a single instance of a model rather than a technique for using multiple instances of a model during inference.

A recent related method is fraternal dropout~\cite{zolna2018fraternal}, which regularizes recurrent neural networks by applying two sets of dropout masks, then using a loss function that promotes similarity in the resulting outputs and invariance to dropout masks. While our proposed method has similar goals, we insert averaging directly into dense layers and use a standard loss function, achieving regularization through ensemble averaging rather than explicitly promoting invariance to dropout masks. Our method also operates on dense layers rather than recurrent neural networks. Finally, our method varies from both model averaging methods and fraternal dropout in that it applies to individual layers rather than entire networks and in that we directly average weights and biases at test time rather than activations or outputs.

\section{Proposed Method}
\label{sec:method}

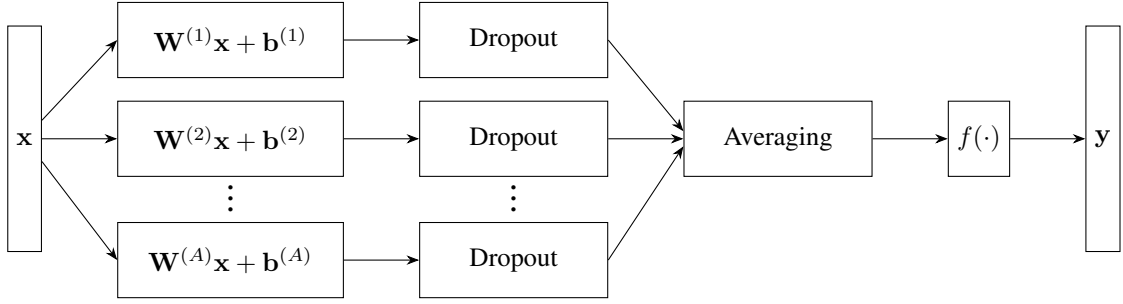
\begin{figure*}[!t]
	\centering
	\begin{tikzpicture}[node distance=1cm]
		\node (x) [vector] {\(\x\)};
		\node (w2) [block, right=of x, minimum width=3cm] {\(\W^{(2)}\x+\b^{(2)}\)};
		\node (w1) [block, node distance=0.3cm, above=of w2, minimum width=3cm] {\(\W^{(1)}\x+\b^{(1)}\)};
		\node (wA) [block, node distance=0.6cm, below=of w2, minimum width=3cm] {\(\W^{(A)}\x+\b^{(A)}\)};
		\path (w2) -- (wA) node[pos=0.3,font=\Large]{\vdots};
		\node (dp1) [block, right=of w1, minimum width=2.5cm] {Dropout};
		\node (dp2) [block, right=of w2, minimum width=2.5cm] {Dropout};
		\node (dpA) [block, right=of wA, minimum width=2.5cm] {Dropout};
		\path (dp2) -- (dpA) node[pos=0.3,font=\Large]{\vdots};
		\node (avg) [block, right=of dp2, minimum width=2.5cm] {Averaging};
		\node (f) [block, right=of avg] {\(f(\cdot)\)};
		\node (y) [vector, right=of f] {\(\y\)};
		\draw [-Stealth] (x) -- (w1.west);
		\draw [-Stealth] (x) -- (w2.west);
		\draw [-Stealth] (x) -- (wA.west);
		\draw [-Stealth] (w1.east) -- (dp1.west);
		\draw [-Stealth] (w2.east) -- (dp2.west);
		\draw [-Stealth] (wA.east) -- (dpA.west);
		\draw [-Stealth] (dp1.east) -- ([yshift=0.1cm]avg.west);
		\draw [-Stealth] (dp2.east) -- (avg.west);
		\draw [-Stealth] (dpA.east) -- ([yshift=-0.1cm]avg.west);
		\draw [-Stealth] (avg.east) -- (f.west);
		\draw [-Stealth] (f.east) -- (y.west);
	\end{tikzpicture}
		\captionsetup{font=small}
	\caption{Diagram of an STE layer using dropout for noise. Inputs are transformed using multiple independent weight matrices and bias vectors, then noise is added using dropout, corresponding outputs are averaged together, and finally the activation function is applied. Note that dropout can additionally be used on the layer output. At test time, weight matrices and biases are averaged to produce a single weight matrix and bias vector.}
	\vspace{-3mm}
	\label{fig:method}
\end{figure*}

To introduce our notation, a standard dense neural network layer with \(M\) inputs and \(N\) outputs implements the following function:
\begin{equation}
	\y = f(\W\x+\b),
\end{equation}
where \(\x\in\Real^M\) is the layer input, \(\W\in\Real^{N\times M}\) is a weight matrix, \(\b\in\Real^N\) is a bias vector, \(f(\cdot)\) is an activation function applied elementwise, and \(\y\in\Real^N\) is the layer output. The weights and biases are learned parameters.

STE layers act as a replacement for dense layers by transforming \(\x\) using \(A\) different weight matrices and bias vectors, applying some noise, and then averaging the results to produce a vector of length \(N\) before applying the activation function. We refer to \(A\) as the averaging factor of the layer. Mathematically, an STE layer with \(M\) inputs and \(N\) outputs is of the form:
\begin{align}
	&\y = f\left(\frac{1}{A}\sum_{i=1}^An(\W^{(i)},\x,\b^{(i)})\right),
\end{align}
where \(n(\W^{(i)},\x,\b^{(i)})\) represents an operation that implements the transformation \(\W^{(i)}\x+\b^{(i)}\), but with some type of noise added during training. The vectors produced by each application of \(n(\cdot)\) are then averaged together. All weight matrices and bias vectors are initialized independently.

At test time, the noise operation \(n(\ldots)\) must become \(\W^{(i)}\x+\b^{(i)}\) or some affine transformation of this. This allows the following property to hold.
\begin{theorem}
	At test time, an STE layer with \(M\) inputs and \(N\) outputs operates as a dense layer with the same number of inputs and outputs.
\end{theorem}
\begin{proof}
	At test time, \(n(\ldots)\) produces \(\C(\W^{(i)}\x+\b^{(i)})+\d\) for some \(\C \in \Real^{N\times N}\) and \(\d \in \Real^N\). The layer output is therefore given by:
	\begin{align}
		\y &= f\left(\frac{1}{A}\sum_{i=1}^A\left[\C(\W^{(i)}\x+\b^{(i)})+\d\right]\right) \\
		&= f\left(\left(\frac{1}{A}\sum_{i=1}^A\C\W^{(i)}\right)\x+\frac{1}{A}\sum_{i=1}^A\left(\C\b^{(i)}+\d\right)\right)
	\end{align}
	which is equivalent to a dense layer with the weight matrix \(\frac{1}{A}\sum_{i=1}^A\C\W^{(i)} \in \Real^{N\times M}\) and the bias vector\\\(\frac{1}{A}\sum_{i=1}^A\left(\C\b^{(i)}+\d\right) \in \Real^N\).
\end{proof}

This means that an STE layer acts as a regularization method, incurring \textit{no additional computational cost at test time} compared to a dense layer of the same dimensions, regardless of the value of \(A\). In contrast, model averaging methods require running inference on multiple model instances, which is computationally costly. It is important to note that the injection of different noise values to each member of the ensemble during training is exactly what makes it impossible to combine weight matrices in a similar fashion during training, ensuring diversity across the ensemble.

We mainly investigate the use of dropout~\cite{hinton2012improving} to provide noise, leading to the following equation during training:
\begin{align}
	&\y = f\left(\frac{1}{A}\sum_{i=1}^A\m^{(i)}\circ(\W^{(i)}\x+\b^{(i)})\right),
\end{align}
where each \(\m^{(i)}\) is an independent length-\(N\) vector of i.i.d. samples from Bernoulli(\(p\)) and \(\circ\) represents element-wise multiplication. This method is illustrated in Figure~\ref{fig:method}. At test time the element-wise multiplication by \(\m^{(i)}\) is replaced with multiplication by \(\E[m^{(i)}_j]=p\). This is equivalent to a dense layer with the weight matrix \(\frac{p}{A}\sum_{i=1}^A\W^{(i)}\) and the bias vector \(\frac{p}{A}\sum_{i=1}^A\b^{(i)}\).

We also explore applying a dropout mask directly to weights, which leads to the following equation:
\begin{align}
	&\y = f\left(\frac{1}{A}\sum_{i=1}^A(\M^{(i)}\circ\W^{(i)})\x+\b^{(i)}\right),
\end{align}
where each \(\M^{(i)}\) is an independent \(N\times M\) matrix of i.i.d. samples from Bernoulli(\(p\)).
This is a generalization of dropconnect~\cite{wan2013regularization}, which represents the special case \(A=1\).
 We differ from \cite{wan2013regularization} and use the same approach as with dropout at test time, replacing multiplication by mask elements with multiplication by \(\E[M^{(i)}_{jk}]\).

\begin{table*}[ht!]
\centering
		\captionsetup{font=small}
		\caption{Comparison of STE layers to standard dense layers in an MLP with and without dropout. Results are averaged over five tests with sample standard deviations provided. Lower loss and higher accuracy are better.}
		\vspace{-2mm}
\begin{adjustbox}{width=\textwidth}
\begin{tabular}{c|cc|cc|cc|cc|cc}
	Dataset & \multicolumn{2}{c}{\makecell{No\\regularization}} & \multicolumn{2}{c}{\makecell{Standard\\dropout}} & \multicolumn{2}{c}{\makecell{STE layers\\using dropout}} & \multicolumn{2}{c}{\makecell{STE layers\\using dropconnect}} \\\hline
	& Loss & Accuracy & Loss & Accuracy & Loss & Accuracy & Loss & Accuracy \\\hline
	CIFAR-10 & 1.351\pmt0.009 & 53.2\pmt0.06 & 1.266\pmt0.01 & 56.4\pmt0.3 & \textbf{1.236}\pmt0.003 & \textbf{57.0}\pmt0.2 & 1.237\pmt0.006 & 56.9\pmt0.4 \\
	CIFAR-100 & 3.224\pmt0.01 & 25.3\pmt0.6 & 2.971\pmt0.009 & 29.8\pmt0.3 & \textbf{2.836}\pmt0.004 & 32.0\pmt0.2 & 2.845\pmt0.003 & \textbf{32.2}\pmt0.1 \\
	SVHN & 0.759\pmt0.009 & 81.5\pmt0.07 & 0.568\pmt0.005 & 85.7\pmt0.1 & \textbf{0.489}\pmt0.002 & \textbf{86.3}\pmt0.09 & 0.507\pmt0.001 & 85.8\pmt0.1 \\
\end{tabular}
\end{adjustbox}
	\label{t:mlp}
\end{table*}

\begin{table*}[ht!]
\centering
		\captionsetup{font=small}
		\caption{Comparison of STE layers to standard dense layers in LeNet-5 with and without dropout. Results are averaged over five tests with sample standard deviations provided. Lower loss and higher accuracy are better.}
		\vspace{-2mm}
\begin{adjustbox}{width=\textwidth}
\begin{tabular}{c|cc|cc|cc|cc|cc}
	Dataset & \multicolumn{2}{c}{\makecell{No\\regularization}} & \multicolumn{2}{c}{\makecell{Standard\\dropout}} & \multicolumn{2}{c}{\makecell{STE layers\\using dropout}} & \multicolumn{2}{c}{\makecell{STE layers\\using dropconnect}} \\\hline
	& Loss & Accuracy & Loss & Accuracy & Loss & Accuracy & Loss & Accuracy \\\hline
	CIFAR-10 & 1.234\pmt0.02 & 57.1\pmt0.8 & 1.166\pmt0.02 & 60.7\pmt0.5 & \textbf{1.030}\pmt0.007 & \textbf{64.6}\pmt0.4 & 1.037\pmt0.02 & 64.0\pmt0.01 \\
	CIFAR-100 & 3.165\pmt0.04 & 24.9\pmt0.8 & 3.038\pmt0.04 & 26.5\pmt0.5 & \textbf{2.937}\pmt0.03 & \textbf{27.2}\pmt0.3 & 3.017\pmt0.01 & 26.0\pmt0.6 \\
	SVHN & 0.505\pmt0.01 & 85.3\pmt0.4 & 0.479\pmt0.003 & \textbf{86.6}\pmt0.2 & \textbf{0.468}\pmt0.004 & \textbf{86.6}\pmt0.2 & 0.472\pmt0.004 & 86.3\pmt0.1 \\
\end{tabular}
\end{adjustbox}
	\label{t:lenet5}
\end{table*}

Whereas dropout and related methods create pseudo-ensembles of models~\cite{bachman2014learning}, STE layers additionally use an explicit ensemble of separate weights and biases. \(A\) different initial sets of weights are sampled, and the use of independent noise for each causes each to follow a different path towards some local minimum. As with bagging methods~\cite{dietterich2000ensemble}, we understand the role of averaging outputs as reducing the variance caused by neural network learning arriving at suboptimal minima. A major difference from bagging is that the weights in STE layers can be directly averaged at test time, as opposed to running multiple models and averaging their outputs. This is possible because STE layers train weight ensembles together in a single layer, causing them to learn in such a way that groups of activations within a neural network correspond to each other and can be averaged. We empirically analyze this effect in the following section and show that the resulting weight matrices remain diverse.

STE layers can also be seen as promoting a network representation that is more robust to different weight initializations and noise than dense layers. By incorporating multiple weight initialization samples and independent applications of noise, the sensitivity to any particular realization is reduced. This has a similar effect to regularization approaches that explicitly promote invariance with respect to dropout masks~\cite{zolna2018fraternal}.

\section{Experiments}
\label{sec:experiments}

We evaluate using STE layers instead of dense hidden layers in a multi-layer perceptron (MLP) as well as replacing fully connected layers in two standard convolutional neural network topologies: LeNet-5~\cite{lecun1998gradient} and AlexNet~\cite{krizhevsky2012imagenet}. These networks are tested on three image classification datasets: CIFAR-10~\cite{krizhevsky2012cifar}, CIFAR-100~\cite{krizhevsky2012cifar}, and SVHN~\cite{netzer2011reading}. Our result metrics are cross-entropy loss and top-1 accuracy in percent. We normalize all datasets to have zero mean and unit standard deviation.

We primarily compare STE layers to dense layers that use standard dropout regularization with \(p=0.5\). We test STE layers using either dropout or dropconnect internally, both with \(p=0.5\), but additionally apply dropout to the layer outputs, again with \(p=0.5\). For comparison, we also report results obtained without any regularization methods applied to the dense layers. We use \(A=8\) for the MLP and LeNet-5 tests, but \(A=2\) for AlexNet tests, since higher values were found to degrade training. Both dense and STE layer weights are initialized using Glorot uniform initialization~\cite{glorot2010understanding} and biases are initialized to zero. We always use a dense layer as the output layer, not an STE layer. The MLP used for testing has two hidden layers with 2\,048 neurons each and ReLU activation functions. Since AlexNet is designed for larger images than those in the datasets used, when training and testing it, we upsample images by a factor of seven using bilinear interpolation.

We train using SGD optimization with Nesterov momentum of 0.9, a minibatch size of 128, and learning rate decay \(\delta\) after each minibatch such that the learning rate at iteration \(n\) is reduced by a factor \(1/(1+\delta n)\). To avoid bias, learning rates and decay rates are tuned to produce optimal results with standard dropout regularization rather than our proposed method. For the MLP, we use an initial learning rate of \num{1e-2} with a decay rate of \num{1e-4} and train for 128 epochs. For LeNet-5, we use an initial learning rate of \num{1e-2} with a decay rate of \num{3e-4} and train for 256 epochs. For AlexNet, we use an initial learning rate of \num{1e-3} with a decay rate of \num{3e-5} and train for 150 epochs. Ten percent of data is randomly held back for validation. For testing, we restore the weights from the epoch with lowest validation loss.

\begin{table*}[ht!]
\centering
		\captionsetup{font=small}
		\caption{Comparison of STE layers to standard dense layers in AlexNet with and without dropout. Results are averaged over five tests with sample standard deviations provided. Lower loss and higher accuracy are better.}
		\vspace{-2mm}
\begin{adjustbox}{width=\textwidth}
\begin{tabular}{c|cc|cc|cc|cc}
	Dataset & \multicolumn{2}{c}{\makecell{No\\regularization}} & \multicolumn{2}{c}{\makecell{Standard\\dropout}} & \multicolumn{2}{c}{\makecell{STE layers\\using dropout}} & \multicolumn{2}{c}{\makecell{STE layers\\using dropconnect}} \\\hline
	& Loss & Accuracy & Loss & Accuracy & Loss & Accuracy & Loss & Accuracy \\\hline
	CIFAR-10 & 0.843\pmt0.02 & 72.4\pmt0.7 & 0.604\pmt0.006 & 80.9\pmt0.3 & \textbf{0.577}\pmt0.01 & \textbf{81.1}\pmt0.5 & 0.600\pmt0.01 & 81.0\pmt0.2 \\
	CIFAR-100 & 2.709\pmt0.03 & 35.5\pmt0.4 & 1.921\pmt0.02 & 50.6\pmt0.5 & \textbf{1.730}\pmt0.02 & \textbf{53.7}\pmt0.4 & 1.858\pmt0.01 & 50.5\pmt0.7 \\
	SVHN & 0.391\pmt0.02 & 90.3\pmt0.3 & 0.265\pmt0.01 & 93.5\pmt0.2 & \textbf{0.231}\pmt0.003 & \textbf{94.0}\pmt0.05 & 0.237\pmt0.005 & 93.9\pmt0.2 \\
\end{tabular}
\end{adjustbox}
	\label{t:alexnet}
\end{table*}

\subsection{Results and Analysis}
\label{sec:results}

\begin{table}[t!]
\centering
		\captionsetup{font=small}
		\caption{The number of parameters trained for the experiments described in this paper. Note that at test time, STE layer parameters are averaged so that there is no additional cost compared to standard dense layers.}
		\vspace{-2mm}
\begin{adjustbox}{width=0.5\textwidth}
\begin{tabular}{c|cc|cc|cc}
	Model & \multicolumn{2}{c}{MLP} & \multicolumn{2}{c}{LeNet-5} & \multicolumn{2}{c}{AlexNet} \\\hline
	& Dense & STE & Dense & STE & Dense & STE \\\hline
	CIFAR-10 & 10.5M & 83.9M & 1.7M & 13.9M & 45.5M & 88.5M \\
	CIFAR-100 & 10.7M & 84.1M & 1.7M & 13.9M & 45.9M & 88.9M \\
	SVHN & 10.5M & 83.9M & 1.7M & 13.9M & 45.5M & 88.5M \\
\end{tabular}
\end{adjustbox}
	\label{t:parameters}
\end{table}

\begin{figure}[t!]
	\centering
	\includegraphics[width=0.5\textwidth]{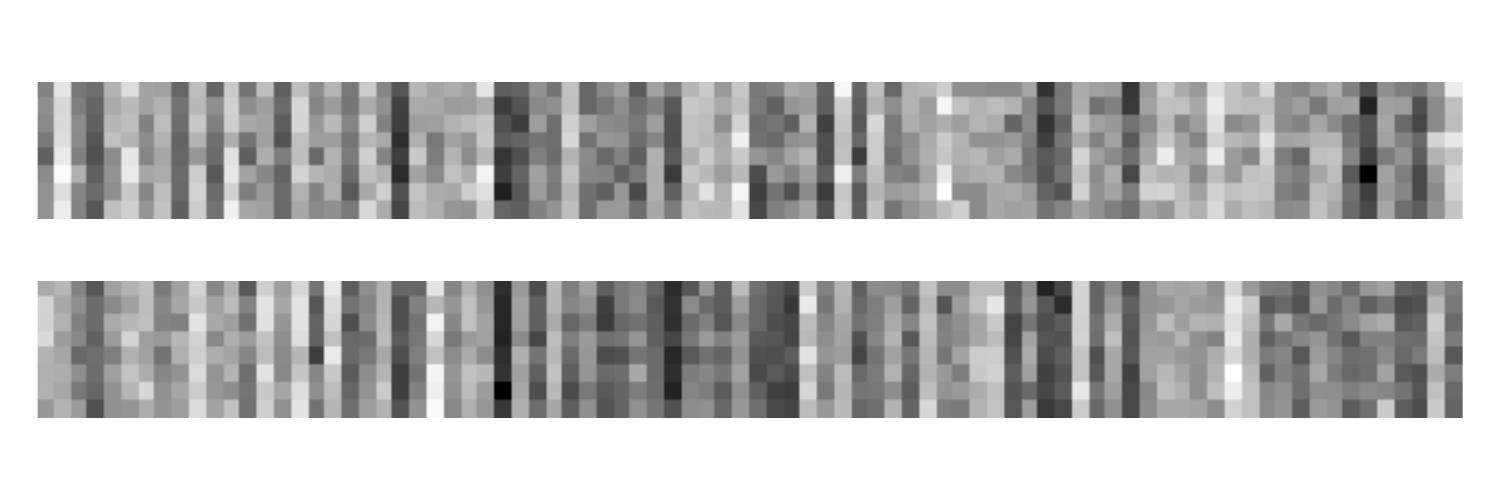}
	\vspace{-5mm}
		\captionsetup{font=small}
	\caption{Two sample activation vectors within an STE layer before averaging, using a trained LeNet-5 model with two test images from CIFAR-10 as inputs. Columns correspond to values that will be averaged together.}
	\label{fig:acts}
	\vspace{-3mm}
\end{figure}

Results on the test partition of each dataset are shown in \Cref{t:mlp,t:lenet5,t:alexnet}. In all cases, lower test loss is achieved using STE layers compared to standard dense layers with dropout regularization. Top-1 accuracy remains the same in some cases, but for the most part, it improves when using STE layers. This improvement can be dramatic, as seen with the 3.9\% increase with LeNet-5 on CIFAR-10. In general, STE layers are shown to improve performance, which we interpret as them providing more effective regularization, allowing the neural networks to generalize more effectively on the given, relatively small, datasets.

For the most part, using dropconnect as the noise method within STE layers shows worse performance than using dropout. However, STE layers with dropconnect still outperform dense layers with dropout. It may be possible to improve its relative performance by tuning \(p\), which we did not investigate in the interest of minimizing tuned hyperparameters in our proposed method. Other noise methods could also be investigated, such as multiplicative Gaussian noise~\cite{srivastava2014dropout}. However, given that dropout is generally thought to be versatile and effective for stochastically training neural networks, it is unclear whether this could provide significant benefits.

Using STE layers allows for much larger networks to be trained with no additional cost at test time. \Cref{t:parameters} shows the number of parameters trained in our experiments when using dense layers versus STE layers. The averaging causes the weights of neurons that are averaged together to become correlated, simplifying the learning task and allowing for weights to be directly averaged together. However, the use of dropout or other noise maintains diversity within each ensemble. This diversity allows weight averaging at test time to produce a regularized network. The correlation between activations is illustrated in \cref{fig:acts}, which shows activations within a trained STE layer before averaging. Activations that are averaged together are clearly correlated, but far from identical, demonstrating that the weight matrices within the ensemble have arrived at different minima within the optimization space.

Note that for the sake of comparison, only the averaging factor was tuned for STE layer-specific tests. Further tuning of optimizer parameters, noise parameters, adding additional regularization, etc. could improve results in practice.

\section{Conclusions}
\label{sec:conclusions}

We find that STE layers can effectively use ensemble averaging to regularize neural networks, with no additional cost at test time. We show consistent improvements compared to dropout regularization using standard neural network topologies and standard image classification datasets. While model averaging methods are common in machine learning, our proposed method differs from most established methods in that it allows weights and biases to be directly averaged at test time, thereby functioning as a regularization method. STE layers are based on dense neural network layers, but there is clear potential for similar methods to work in neural networks with different kinds of layers, including convolutional neural networks, recurrent neural networks, residual neural networks, and neural networks with attention mechanisms.

% \clearpage
\bibliographystyle{IEEEbib}
\bibliography{mybibfile}

\begin{thebibliography}{10}

\bibitem{hinton2012improving}
Geoffrey~E Hinton, Nitish Srivastava, Alex Krizhevsky, Ilya Sutskever, and
  Ruslan~R Salakhutdinov,
\newblock ``Improving neural networks by preventing co-adaptation of feature
  detectors,''
\newblock {\em arXiv preprint arXiv:1207.0580}, 2012.

\bibitem{labach2019survey}
Alex Labach, Hojjat Salehinejad, and Shahrokh Valaee,
\newblock ``Survey of dropout methods for deep neural networks,''
\newblock {\em arXiv preprint arXiv:1904.13310}, 2019.

\bibitem{wardefarley2014empirical}
David Warde-Farley, Ian~J. Goodfellow, Aaron Courville, and Yoshua Bengio,
\newblock ``An empirical analysis of dropout in piecewise linear networks,''
\newblock in {\em Proceedings of the International Conference on Learning
  Representations (ICLR)}, 2014.

\bibitem{baldi2013understanding}
Pierre Baldi and Peter~J Sadowski,
\newblock ``Understanding dropout,''
\newblock in {\em Advances in Neural Information Processing Systems 26}, pp.
  2814--2822. Curran Associates, Inc., 2013.

\bibitem{bachman2014learning}
Philip Bachman, Ouais Alsharif, and Doina Precup,
\newblock ``Learning with pseudo-ensembles,''
\newblock in {\em Advances in Neural Information Processing Systems 27}, pp.
  3365--3373. Curran Associates, Inc., 2014.

\bibitem{dietterich2000ensemble}
Thomas~G. Dietterich,
\newblock ``Ensemble methods in machine learning,''
\newblock in {\em Multiple Classifier Systems}, Berlin, Heidelberg, 2000, pp.
  1--15, Springer Berlin Heidelberg.

\bibitem{masoudnia2014mixture}
Saeed Masoudnia and Reza Ebrahimpour,
\newblock ``Mixture of experts: a literature survey,''
\newblock {\em Artificial Intelligence Review}, vol. 42, no. 2, pp. 275--293,
  Aug 2014.

\bibitem{zolna2018fraternal}
Konrad {\.Z}o{\l}na, Devansh Arpit, Dendi Suhubdy, and Yoshua Bengio,
\newblock ``Fraternal dropout,''
\newblock {\em arXiv preprint arXiv:1711.00066}, 2018.

\bibitem{wan2013regularization}
Li~Wan, Matthew Zeiler, Sixin Zhang, Yann Le~Cun, and Rob Fergus,
\newblock ``Regularization of neural networks using dropconnect,''
\newblock in {\em International conference on machine learning}, 2013, pp.
  1058--1066.

\bibitem{lecun1998gradient}
Yann LeCun, L{\'e}on Bottou, Yoshua Bengio, and Patrick Haffner,
\newblock ``Gradient based learning applied to document recognition,''
\newblock in {\em Proceedings of the IEEE}, 1998, pp. 2278--2324.

\bibitem{krizhevsky2012imagenet}
Alex Krizhevsky, Ilya Sutskever, and Geoffrey~E Hinton,
\newblock ``Imagenet classification with deep convolutional neural networks,''
\newblock in {\em Advances in neural information processing systems}, 2012, pp.
  1097--1105.

\bibitem{krizhevsky2012cifar}
Alex Krizhevsky,
\newblock ``Learning multiple layers of features from tiny images,''
\newblock {\em University of Toronto}, 05 2012.

\bibitem{netzer2011reading}
Yuval Netzer, Tao Wang, Adam Coates, Alessandro Bissacco, Bo~Wu, and Andrew~Y.
  Ng,
\newblock ``Reading digits in natural images with unsupervised feature
  learning,''
\newblock in {\em NIPS Workshop on Deep Learning and Unsupervised Feature
  Learning}, 2011.

\bibitem{glorot2010understanding}
Xavier Glorot and Yoshua Bengio,
\newblock ``Understanding the difficulty of training deep feedforward neural
  networks.,''
\newblock in {\em Proceedings of the International Conference on Artificial
  Intelligence and Statistics}. AISTATS, 2010, vol.~9, pp. 249--256.

\bibitem{srivastava2014dropout}
Nitish Srivastava, Geoffrey~E Hinton, Alex Krizhevsky, Ilya Sutskever, and
  Ruslan Salakhutdinov,
\newblock ``Dropout: a simple way to prevent neural networks from
  overfitting.,''
\newblock {\em Journal of Machine Learning Research}, vol. 15, no. 1, pp.
  1929--1958, 2014.

\end{thebibliography}

\end{document}